\documentclass[conference]{IEEEtran}
\IEEEoverridecommandlockouts
\usepackage{cite}
\usepackage{amsmath,amssymb,amsfonts}
\usepackage{mathtools}
\usepackage{amsthm}
\usepackage{amsmath}
\usepackage{enumitem}
\usepackage{graphicx}
\usepackage{textcomp}
\usepackage{xcolor}
\usepackage{verbatim}
\usepackage{graphicx}
\usepackage{subfig}
\usepackage{url}
\usepackage{wrapfig}

\def\BibTeX{{\rm B\kern-.05em{\sc i\kern-.025em b}\kern-.08em
    T\kern-.1667em\lower.7ex\hbox{E}\kern-.125emX}}

\newcommand{\isep}{\mathrel{{.}\,{.}}\nobreak}

\newtheorem{theorem}{Theorem}[section]
\newtheorem{conjecture}{Conjecture}[section]

\newtheorem{remark}[theorem]{Remark}
\theoremstyle{definition}
\newtheorem{definition}{Definition}[section]

\usepackage{xspace} 
\newcommand{\pn}{$PN$\xspace}
\newcommand{\dn}{$DN$\xspace}
    
\begin{document}
% =============== COPYRIGHT =============
\IEEEoverridecommandlockouts
%\IEEEpubid{\makebox[\columnwidth]{978-1-6654-3886-5/21/\$31.00 \copyright 2021 European Union \hfill} \hspace{\columnsep}\makebox[\columnwidth]{ }}

%\title{Boosting Proof Number Search via Domain Knowledge Incorporation for the Game of Gobang\\
\title{Towards solving the 7-in-a-row game\\
\thanks{This work was supported by the European Union, co-financed by
  the European Social Fund (EFOP-3.6.3-VEKOP-16-2017-00002), the ERC DYNASNET Grant (ERC-2018-SYG 810115), as well as by the Hungarian National Excellence Grant
  2018-1.2.1-NKP-00008. It was also supported by the Hungarian Ministry of
  Innovation and Technology NRDI Office within the framework of the
  Artificial Intelligence National Laboratory Program.}
}

\author{\IEEEauthorblockN{1\textsuperscript{st} Domonkos Czifra}
\IEEEauthorblockA{\textit{Alfr\'{e}d R\'{e}nyi Institute of Mathematics} \\
Budapest, Hungary
}
\and
\IEEEauthorblockN{2\textsuperscript{nd} Endre Cs\'{o}ka}
\IEEEauthorblockA{\textit{Alfr\'{e}d R\'{e}nyi Institute of
    Mathematics} \\
Budapest, Hungary
}
\and
\IEEEauthorblockN{3\textsuperscript{rd} Zsolt Zombori}
\IEEEauthorblockA{\textit{Alfr\'{e}d R\'{e}nyi Institute of
    Mathematics} \\
\textit{E{\"o}tv{\"o}s Lor\'{a}nd University}\\
Budapest, Hungary
}
\and
\IEEEauthorblockN{4\textsuperscript{th} G\'{e}za Makay}
\IEEEauthorblockA{\textit{University of Szeged} \\
Szeged, Hungary
}
}

\maketitle

\begin{abstract}
Our paper explores the game theoretic value of the 7-in-a-row game. We reduce the problem to solving a finite board game, which we target using Proof Number Search. We present a number of heuristic improvements to Proof Number Search and examine their effect within the context of this particular game. Although our paper does not solve the 7-in-a-row game, our experiments indicate that we have made significant progress towards it.
\end{abstract}

\begin{IEEEkeywords}
proof number search, (m,n,k)-games, artificial intelligence, game theoretic value
\end{IEEEkeywords}

\section{Introduction}
\label{sec:introduction}

Our paper explores how Proof Number Search (PNS) can be adapted to prove the infinite 7-in-a row game, whose game theoretic value has long been open. This game belongs to the family of $(m,n,k)$--games -- a generalisation of Gomoku and Tic-tac-toe. In these games two players take turns in placing a stone of their color on a rectangular board of size $m \times n$ (where both $m$ and $n$ can be
infinite), the winner being the player who first gets $k$ stones of
their own color in a row, horizontally, vertically, or
diagonally. We focus on a weak variant of $(m,n,k)$--games, called the maker-breaker setup, where the second player (breaker) cannot win by collecting $k$ stones in a row, hence its objective is to prevent the first player (maker) from winning.

On the theoretical side, we present a tiling technique that can be used to prove that breaker wins on an infinite board by partitioning the board into finite pieces and generalising the breaker strategy on the small boards to the infinite one.

Afterwards, we search for a breaker win strategy on the finite board using PNS, a technique that has already been successfully applied to several board games, e.g. Gomoku, Hex, and Go. PNS benefits from the non-uniform branching factor of the AND-OR game tree. Exploiting domain specific knowledge of $(m,n,k)$--games, we develop several methods that increase this non-uniformity, which further reduces the search space and increases the computational gain of PNS compared to Alpha-Beta~\cite{Bundy1984} pruning in many scenarios. 

Our methods can be grouped into three categories. The first category concerns the reduction of the search space, such as early recognition of winning states, recognition of mandatory moves,  and partitioning of the board. The second category is about identifying isomorphic states. Finally, the third category of heuristics guides the traversal of the search space by overriding the static initialization rule of the proof and disproof number values with heuristic ones. Our initialization uses a simple combination of heuristic features and parameters learned from previously proven states, foreshadowing the potential of enhancing other board game solvers with machine learning. Our paper presents a quantitative evaluation of the effect of these changes on the search space.

The 7-in-a-row game corresponds to the $(\infty, \infty, 7)$--game and our paper presents work geared towards proving the conjecture that its game theoretic value is a draw. Our contribution can be summarized as follows:
\begin{itemize}
    \item We present a tiling technique that allows us to reduce the infinite board $(\infty,\infty,7)$--game to (infinitely many) independent finite $(4,n,7^{tr})$--games for some fixed $n$.
    \item We incorporate various search heuristics specific to $(m,n,k)$--games into PNS. In particular, we introduce three methods that are -- to the best of our knowledge -- novel: 1) isomorphy detection, 2) breaking the board into components, 3) heuristic proof number and disproof number initialization.
    \item  We empirically evaluate each of our methods.
    \item We prove that the $(4,n,7^{tr})$--game is maker win for $n \leq 14$. Our experiments, however, suggest that as $n$ increases, the closer we get to a breaker win situation, leading to the conjecture that there is a $n_0$ where the game theoretic value flips, i.e. the game is maker win for $n < n_0$ and breaker win for $n \geq n_0$.
\end{itemize}

%As a minor result, we can solve the weak $(m, n, 4)$-game for every m
%and n, and we demonstrate that we are very close to prove the
%long-standing open conjecture that the weak $(\infty, \infty, 7)$-game is a
%draw.

\section{Background and Related Work}
\label{sec:background}
%% 0.75 page

\subsection{$(m,n,k)$--games}
$(m,n,k)$--games are played on an $m \times n$ board, where two players take turns in marking one of the empty fields of the board. The player, who can collect $k$ marks in a row (horizontally, vertically, or diagonally) wins the game. $(m,n,k)$--games belong to \emph{positional games} \cite{BECK1981117}, defined more abstractly, as follows. Let $H=(V,E)$ be a hypergraph. The two players take turns to mark a node with their color and the winner is the player who first colors an entire hyperedge with his color. In particular, an $(m,n,k)$--game is a positional game where $V:=\{v \mid v \in m \times n \text{ board}\}$ and $E$ contains all horizontal, vertical and diagonal lines of length $k$. 

A player is said to have a \emph{winning strategy} if it can always win, regardless of the opponent's strategy. A player has a \emph{draw strategy} if the other player does not have a winning strategy. Accordingly, the \emph{game theoretic value} of a game can be 1) \emph{first player win}, 2) \emph{second player win}, or 3) \emph{draw}.

The \emph{strategy stealing} argument can be used to show that the second player cannot have a winning strategy: if it had, the first player could start with an arbitrary move and then mimic the second player's strategy to win, leading to contradiction. This motivates a weaker version of the $(m, n, k)$--games,
called \emph{maker-breaker} setup, in which the aim of the second player (breaker) is to prevent the first player (maker) from winning, i.e., breaker is not rewarded by collecting an entire hyperedge. The game theoretic value in the maker-breaker setup is 1) \emph{maker win} if the first player has a winning strategy or 2) \emph{breaker win} otherwise.

The following observations are easy to prove. If maker wins some $(m,n,k)$--game, then its winning strategy directly applies to any games with greater $m$ or $n$, or smaller $k$. If breaker wins some $(m,n,k)$--game, then it also wins if $m$ or $n$ are decreased, or $k$ is increased. 
If first player wins some $(m,n,k)$--game, then it also holds that maker wins that game. Conversely, if breaker wins, then the game is a draw. It is, however, possible that the maker-breaker variant is a maker win but the original game is a draw: for example, the $(3,3,3)$--game (or Tic Tac Toe).

Several games have been proven to be draws, e.g. the $(5,5,4)$--game by \cite{BerlekampElwynR1983Wwfy}, the $(6,6,5)$--game by \cite{UITERWIJK200043}, and the $(7,7,5)$--game by \cite{gomoku775}. Recently, \cite{HSU202079} prove that the $(8,8,5)$-game is a draw as well. On the other hand, \cite{Allis1994SearchingFS} show that first player wins the $(15,15,5)$--game, also called \emph{Gomoku}.\footnote{Gomoku is also played on a $19 \times 19$ board, but the $(19, 19, 5)$--game is still unsolved.}

In the maker-breaker setup, a maker-color in some square corresponds to removing that square from all hyperedges. In contrast, a breaker-color in a square corresponds to removing all hyperedges containing that square. Hence, each move can be seen to make the hypergraph smaller. We introduce \emph{l-lines}, to characterize the active parts of the board:

\begin{definition}[l-line]
An \emph{l-line} is a hyperedge which contains no breaker-colored squares and exactly $l$ empty squares.
\end{definition}

We use l-lines to define an aggregate statistic board measure, called \emph{potential}, which will be crucial for developing good search heuristics. \cite{BECK1981117} already introduces potential and it is used in several works on $(m,n,k)$--games.

\begin{definition}[Potential]
Suppose board $b$ contains $x_l$ different $l$-lines for $l \in 1 \dots k$. The \emph{potential} of $b$ is
$$pot(b) = \sum_{l=1}^k x_l \cdot 2^{-(l-1)}$$
\end{definition}

\subsection{7-in-a-row game}
The 7-in-a-row game is an $(m,n,k)$--game, where $m,n=\infty$, representing the board $\mathbb{Z} \times \mathbb{Z}$ and $k=7$, hence it can be witten as the $(\infty,\infty,7)$--game. \cite{BerlekampElwynR1983Wwfy} proves that the $(\infty,\infty,9)$--game is a breaker win and \cite{inf_inf_8} proves that the $(\infty,\infty,8)$--game is a breaker win, as well. \cite{Allis1994SearchingFS} proves that first player wins the $(15,15,5)$--game, which implies that maker wins the $(\infty,\infty,5)$--game.\footnote{However, it does not imply that the $(\infty,\infty,5)$--game is first player win, which is still open.} These results imply that the $(\infty,\infty,k)$--game is maker win for $k \leq 5$ and breaker win for $k \geq 8$. The cases $k \in \{6,7\}$ are unknown, generally conjectured to be both breaker win.
Our primary objective in this paper is to build techniques and intuition towards proving that the $(\infty,\infty,7)$--game is a breaker win, and hence a draw.

\subsection{Proof Number Search}
\emph{Proof number search} (PNS) is a widely used alogorithm for solving games \cite{PNS_base}, \cite{hex_8x8}. It is based on \emph{conspiracy number search} \cite{conspiracy}, which proceeds in a minimax tree into the direction where the least number of leafs must be changed in order to change the root's minimax value by a given amount. PNS follows a similar strategy, applied to AND/OR trees: it proceeds into the direction where the given node can be proven with the least effort.

\begin{definition}[proof/disproof number]
Given a rooted AND/OR tree with root $r$. The proof/disproof number (\pn/\dn) is the minimum number of descendent leafs, which need to be proven/disproven in order to prove/disprove $r$. If $r$ is a leaf 
%($|T| = 1$)
then the proof and disproof numbers are by definition 1.
\end{definition}

It follows, that
\begin{align*}
\left\{
	\begin{array}{ll}
		\pn=0,\, \dn=\infty  & \mbox{if \textit{r} is proven}\\
		\pn=\infty,\, \dn=0  & \mbox{if \textit{r} is disproven}\\
	\end{array}
\right.
\end{align*}

In order to prove an OR node, we only need to prove one of it's children, but to disprove it, we need to disprove all of it's children.
The contrary is true for AND nodes.
This implies that the proof numbers and disproof numbers can be computed recursively:
\begin{align*}
PN(r) =
\left\{
	\begin{array}{ll}
		\displaystyle \min_{ch: \:  children(r)}\pn(ch)  & \mbox{if \textit{r} is an OR node}\\
		\displaystyle \sum_{ch: \:  children(r)}\pn(ch)  & \mbox{if \textit{r} is an AND node}\\
	\end{array}
\right.
\end{align*}
\begin{align*}
DN(r) =
\left\{
	\begin{array}{ll}
		\displaystyle \sum_{ch: \:  children(r)}\dn(ch)  & \mbox{if \textit{r} is an OR node}\\
		\displaystyle \min_{ch: \:  children(r)}\dn(ch)  & \mbox{if \textit{r} is an AND node}\\
	\end{array}
\right.
\end{align*}

Plain PNS is a best-first algorithm: it selects iteratively the most promising leaf, and extends it, until the root is proven/disproven.
Finding the most promising leaf is computed by starting from the root, and choosing iteratively the branch which may need the least effort to prove.
We measure this effort by the \pn,\dn values: we choose the child with minimal \pn value at OR nodes, and minimal \dn value at AND nodes. For further details of PNS and its variants see \cite{PNS_variants}.

\subsection{Search space reduction techniques}
Solving a game typically involves traversing a large space, hence success is heavily dependent on techniques that reduce search. In the following, we summarize the heuristics that have been successfully applied to solve $(m,n,k)$--games.

In any game with confluent branches, i.e., when different move sequences can result in identical game state, one can save a lot of computation by collapsing identical states, i.e., turning the search tree into a directed acyclic graph (DAG). This is typically implemented using a transposition table (see e.g. \cite{pns_transposition}). \cite{Allis1994SearchingFS} and \cite{HSU202079} both report using transposition tables.

Threat space search \cite{Allis1994SearchingFS} revolves around the observation that in situations where the non-current player has an immediate win option -- a \emph{threat} -- the current player is forced to block that move, hence its effective branching factor is reduced to 1. Any strategy that creates threats has the potential to greatly reduce the proof search effort. While threats have proven to be very useful for proving first player (maker) victory, their use is less clear for proving draw.

A relevance-zone, also called R-zone is a generalisation of threats and captures the part of the game board in which a player has to move into in order to win. Identifying R-zones often allows for reducing the branching factor. R-zones have been used to speed up proof search both in maker win games \cite{connect_mnkpq,relevance_zone} and breaker win games \cite{HSU202079}.

Heuristic ordering of game states allows for first exploring more promising moves. A more promising move is more likely to lead to victory, after which there is no need to explore alternative moves. \cite{Allis1994SearchingFS} use simple manual heuristics to select the top best moves of first player. \cite{HSU202079} uses board potential to order moves. In the context of PNS, heuristic ordering can be implemented via better initialization of \pn/\dn values in leaf nodes: instead of 1 it can be an estimate of how many descendents of the leaf need to be proven/disproven in order to prove/disprove it. A perfect estimate would ensure that PNS finds the smallest solution tree without extending any node outside of the solution. While such perfect estimate is infeasible, in many scenarios we have a more accurate estimate than 1. \cite{Allis1994SearchingFS} use $1+d/2$ for initial \pn/\dn values, where $d$ is the depth of the leaf, encouraging more shallow search. %In Section~\ref{Heuristic_PN_DN} we introduce more refined initialization heuristics.

Another simple but important heuristic is to eliminate squares that are not contained in any l-line: all hyperedges that contain such squares are already blocked by breaker so neither player benefits from moving there.

Pairing strategies \cite{Hales1987, BerlekampElwynR1983Wwfy,UITERWIJK200043,inf_inf_8} yield a useful tool for proving that a board position is a draw or breaker win. A pairing strategy is a set of pairwise disjunct pairs of vertices, such that each hyperedge contains at least one pair. Such pairing does not necessarily exist, but if it does, it can be shown that breaker can win the game: it can block all hyperedges by always marking vertex $v$ after maker has selected the pair of $v$. Identifying a pairing strategy is a useful method for early termination of proof search. When a pairing strategy is not available, partial pairings can be used to eliminate parts of the hypergraph. Given hypergraph $(V, E)$ and some subset $V_{pair} \subseteq V$, let $E_{pair} \subseteq E$ denote the hyperedges restricted to $V_{pair}$. If $(V_{pair}, E_{pair})$ contains a pairing strategy, then the proof theoretical value of $(V, E)$ is the same as that of $(V - V_{pair}, E - E_{pair})$ (\cite{HSU202079}).

Vertex domination was proposed in \cite{HSU202079}. We say that vertex $v_i$ dominates vertex $v_j$ if $E(v_i) \supseteq E(v_j)$. \cite{HSU202079} proves that if $v_i$ dominates $v_j$, then we can always select $v_i$ instead of $v_j$. Furthermore, if two vertices mutually dominate each other, then they form a partial pairing and hence can be removed, along with the containing hyperedges.

Board potential provides another powerful technique to discover that breaker has won the game.

\begin{theorem}
\label{thm:breaker_win}
Consider board $b$, with breaker moving next. If $pot(b) < 1$ then breaker wins in $b$.
\end{theorem}

\begin{proof}
Consider all lines $n_1 \dots n_r$ containing some square $s$. Suppose their lengths are $l_1 \dots l_r$, respectively. The contribution of these lines to the total potential is $cont(s) = \sum_{i=i}^{r} 2^{-(l_i-1)}$. If maker moves to $s$, all lines turn shorter by one, doubling their contribution to $2\cdot cont(s)$. If, on the other hand breaker moves to $s$, the lines become dead, making their contribution $0$. In either case, the change in potential is the same ($cont(s)$), but with different sign. 

If breaker comes next on board $b$, he can always select the square $s$ with the largest corresponding potential contribution. Any square $s'$ that maker subsequently selects has at most the same contribution, i.e., $cont(s) \geq cont(s')$. This means that the potential of the resulting board $b'$ cannot increase: $pot(b') \leq pot(b) - cont(s) + cont(s') \leq pot(b)$. Breaker hence has a strategy that ensures that the potential is monotonic decreasing in every two steps. This, combined with the assumption that $pot(b) < 1$, entails that for any successor board $b'$ of $b$ $pot(b') < 1$. 
Assume, for contradiction that maker wins. This can only happen if his last move was into a 1-line. However, the potential contribution of that line is $1$, contradicting the assumption that breaker can always ensure that the potential is strictly less than $1$.
\end{proof}

\section{Reduction of the $(\infty,\infty,7)$--game to the Finite
  ${(4,n,7^{tr})}$--game}
\label{sec:reduction}

We aim to prove that the $(\infty,\infty,7)$--game is a draw by proving that breaker wins this game.
Finding a breaker strategy on an infinite board can be difficult, but in some cases breaker can partition its strategy into pieces. Such partition involves partitioning the board itself and dealing with each partition independently: when maker colors a node in one of the partitions, breaker answers in the same partition, regardless of the other partitions.
Beyond node partitioning, we also have to partition the edges of the hypergraph: for every hyperedge there should be a hyperedge in one of the partitions, which is a subset of the initial hyperedge. Formally: 

\begin{theorem}
\label{thm:reduction}
Let $H = (V, E)$ be a hypergraph and let $V_1, V_2, \dots$ denote a (possibly infinite) partitioning of its vertex set. Let $E_1, E_2, \dots$ denote edges defined on $V_1, V_2, \dots$, respectively, such that $\forall e \in E (\exists i (\exists e' \in E_i (e' \subseteq e)))$. If breaker wins in each $(V_i, E_i)$ then it wins $(V, E)$ as well.
\end{theorem}

\begin{remark}
Note that in case breaker cannot win in some of the partitions, this does not imply that maker wins $(V,E)$.
\end{remark}

\begin{proof}
Consider an edge $e \in E$. We know that there is an edge $e'$ contained among the edges $E_k$ of some subgraph $(V_k, E_k)$ such that $e' \subseteq e$.  Each time maker moves, it colors a vertex that is contained in exactly one $V_i$. Breaker can respond by following his winning strategy in the same subgraph $(V_i, E_i)$. This ensures that in each subgraph $(V_i, E_i)$ breaker will eventually block all hyperedges. Hence breaker will eventually block $e'$ as well, which implies that it also blocks $e$. 
\end{proof}

We  partition the board into finite $(4,n)$ blocks with nodes $x_{(i,j)}$ $i\in [1\isep 4]$, $j \in [1\isep n]$, with the following hyperedges:
    
\noindent {\bf Horizontal edges}, for $i \in [1\isep 4]$
        \begin{align}
        &\{x_{(i,1)},x_{(i,2)},x_{(i,3)},x_{(i,4)}\} \label{horizontal-4} \\
        &\{x_{(i,n-3)},x_{(i,n-2)},x_{(i,n-1)},x_{(i,n)}\} \label{7-line-def}\\
        &\{x_{(i,j)},\dots,x_{(i,j+6)}\} & j \in [2 \isep n-7]
    \end{align}
    
\noindent {\bf Vertical edges}:
        \begin{align}
        &\{x_{(1,j)},x_{(2,j)},x_{(3,j)},x_{(4,j)}\} & j \in [1 \isep n] \label{vertical-4} 
        \end{align}
    
\noindent {\bf Diagonal edges}:
        \begin{align}
        &\{x_{(i+1,1)},x_{(i+2,2)},x_{(i+3,3)},x_{(i+4,4)}\} & i \in [0 \dots n-4] \label{cross-4-1}\\
        &\{x_{(i-1,1)},x_{(i-2,2)},x_{(i-3,3)},x_{(i-4,4)}\} & i \in [n+1 \dots 5] \label{cross-4-2}
        \end{align}
        
\noindent {\bf Extra edges}:
        \begin{align}
        &\{x_{(3,1)},x_{(2,2)},x_{(1,3)}\}, \{x_{(2,1)},x_{(3,2)},x_{(4,3)}\}, \notag \\
        &\{x_{(3,n-3)},x_{(2,n-2)},x_{(1,n-1)}\}, \notag \\ 
        &\{x_{(2,n-3)},x_{(3,n-2)},x_{(4,n-1)}\} \label{three-rules}\\
        & \{x_{(2,1)},x_{(1.2)}\}, \, \{x_{(n-2,1)},x_{(n-1,2)}\} \label{two-rules}
        \end{align}

For the visualization see Fig. \ref{fig:disproof_setup_and_lines}.
\begin{theorem}[]
\label{thm:reduction_4k}
Let us partition our infinite hypergraph $(V,E)$ into $(4,n)$ blocks, and define the above hyperedges on the blocks.
Then for every hyperedge $e\in E$ there exist a block $(V_{i,j},E_{i,j})$, which contains a hyperedge $f\in E_{i,j}$, which is the subset of $e$. 
\end{theorem}

\begin{proof}
For any 7-line $l \in E$ one of the following holds:
\begin{enumerate}[label=\textbf{C.\arabic*}]
    %\item $\exists i,j (\forall v \in l (v \in V_{i,j})) $
    \item All $v \in l$ are contained in a single block. \label{inner-7}
    \item $l$ crosses at least 2 blocks,  and it contains $4$ vertices in one of them. \label{side-4}
    \item $l$ crosses 3 blocks, and has at most 3 nodes in each block. That's only possible in the corners. \label{corner-case}
\end{enumerate}

\noindent In \ref{inner-7}, $l$ must be horizontal (because a block has 4 rows) and hence it is covered by (\ref{7-line-def}) obviously.

\noindent In \ref{side-4}, $l$ has at least 4 nodes in one block and is covered by (\ref{horizontal-4}), (\ref{vertical-4}), (\ref{cross-4-1}) or (\ref{cross-4-2}).

\noindent In \ref{corner-case}, consider the $4 \times \infty$ region, where $l$ has four nodes. If all these nodes are in one block, then (\ref{cross-4-1}) or (\ref{cross-4-2}) covers $l$. Otherwise, this 4-line crosses the horizontal separator between neighbouring blocks (see figure \ref{fig:corner_lines}). If the separator splits the nodes 2-2, the upper two is covered by (\ref{two-rules}). Otherwise, the split is 3-1 and the 3-subline is covered by (\ref{three-rules}).
\end{proof}

\begin{figure}
    \centering
    \includegraphics[scale=0.4]{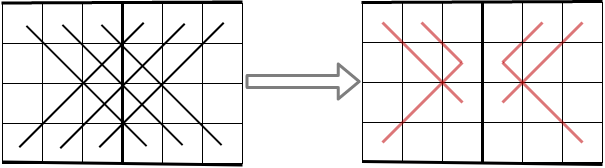}
    \caption{There are 6 4-lines which cross the horizontal border of two neighbouring blocks. We choose a subline, and we add it to the respective block. We choose the larger one by the 3-1 split, and we choose the upper ones by the 2-2 split. In the left figure you can see, how a 4-line can cross the horizontal border, and in the right the respective subparts, which we add to the respective blocks.}
    \label{fig:corner_lines}
    \vspace{-10pt}
\end{figure}

\begin{definition}
Let $(4,n,7^{tr})$--game denote the maker-breaker game on hypergraph $H=(V,E)$, $H=\{ v_{(i,j)} | \, i\in [1 \dots 4],\, j \in [1 \dots n] \}$ and (truncated) hyperedges defined above.
\end{definition}

\begin{conjecture}
\label{thm:csoka}
There is a breaker win strategy for the $(4,n,7^{tr})$--game if $n$ is large enough. 
\end{conjecture}

It is an easy consequence of Conjecture~\ref{thm:csoka} and Theorems \ref{thm:reduction} and \ref{thm:reduction_4k} that the $(\infty,\infty,7)$--game is a draw. Hence, we have reduced our proof task from the infinite $(\infty,\infty,7)$--game to finding some $n$ for which the finite $(4,n,7^{tr})$--game is breaker win. In the following we present our results towards proving the latter conjecture.

We support Conjecture~\ref{thm:csoka} by the following. We define the table $(4, \mathbb{N}, *)$ as follows. The table is $[1..4] \times [1..\infty)$, and we add the 7 extra edges (including \eqref{horizontal-4}) in the only end.

\begin{conjecture}
\label{thm:csoka2}
There is a breaker win strategy for the $(4, \mathbb{N}, *)$--game. 
\end{conjecture}

%This conjecture is based on the following experience for planar game boards \footnote{In general, it should be true for hypergraphs with subexponential growth.}.
This conjecture is based on the following experience for planar game boards (In general, it should be true for hypergraphs with subexponential growth).
If Maker can win, then he can win in a small number of moves and with a very limited-size proof tree.
Also, if we see that breaker can choose between too many options which do not lead to a limited-size win for maker, that the state is a breaker win. 
All these are even more true if the game board has a small area in which the maker has better opportunities. Like in $(4, N, 7^{tr})$.
Here, we experienced that maker has a strong attack in the first few columns, but if the breaker finds the proper defense in the first few moves, then every possible attack ends up in a state where maker seems to have much worse attacking power compared to the starting state, and from which every reasonable greedy breaker strategy seems to be enough for preventing maker from winning.
We were able to formally prove that there is no maker win if maker can use the first 13 coloums.
When we considered tiny modifications of the board, if maker could win, then he was able to win using only the first 6 coloumns.

% ==== Old proof of conjecture ===
%This conjecture is based on the experience that even if we make some initial moves, then within a small number of moves either maker wins or breaker gets the upper hand. The latter case means that after surviving the first few moves, a simple greedy potential-minimizing breaker can block all lines and maker has virtually no chance. And of course, we tested that maker cannot win in a large fraction of the $(4, \mathbb{N}, *)$--table, even with some extra advantage for maker at the other end.

%This suggests that if $n$ is large enough, then maker's attacking power is the same as in two disjoint $(4, \mathbb{N}, *)$--tables. Surprisingly, $n =  14$ is still not large enough: maker can use his threat at each end to build a bit weaker attacking position closer to the middle of the board, and then these two attacking positions combined is enough for a maker win in the middle. It can be seen in this link: \url{http://renyi-amoba.herokuapp.com}.
%\href{http://google.com}{\color{red}{Google}}.

\section{Proof Number Search for Solving the $(4,n,7^{tr})$--game}
\label{sec:pns}

As we have seen in Section~\ref{sec:reduction}, any breaker win strategy for the $(4,n,7^{tr})$--game for some $n$ can be extended to a breaker win strategy in the $(\infty,\infty,7)$--game. The board of the $(4,n,7^{tr})$--game is finite (with $4 \cdot n$ moves), and so are the set of possible move sequences, hence it may be possible to find its game theoretic value using search algorithms. Nevertheless, naive and exhaustive search remains infeasible even for small $n$ values: there are roughly $(4 \cdot n)!$ valid games.\footnote{The real number of games is somewhat less since as soon as maker wins, we can terminate the game.} 

In the following, we instrument PNS to solve the $(4,n,7^{tr})$--game. One main strength of PNS is that it makes no game specific assumptions and can be used for any finite game. It can be seen as a refinement of Alpha/Beta~\cite{Bundy1984} pruning, in that in each extension step it takes a global look at the whole search tree and selects the branch that requires the least number of nodes to prove in order to prove the root node. This can result in the pruning of large parts of the search tree. Nevertheless, as we shall see in Section~\ref{sec:experiments}, the search space remains prohibitively large for plain PNS. Hence, we introduce several methods that reduce computation.
%Some of these methods are game agnostic, while others exploit particularities of the $(m,n,k)$--games. 
Some of our methods are well known or refinements of earlier ideas. Those that are novel to the best of our knowledge are the following: 1) isomorphy detection, 2) breaking the board into components and 3) heuristic \pn/\dn initialization. 

\subsection{Early recognition of winning states}
The search space grows exponentially with the number of valid steps, hence it is crucial to realise once the game has been decided and no further search is necessary. Theorem~\ref{thm:breaker_win} yields a sufficient condition for terminating the game, which we refer to as {\bf Breaker win stop}.
There is a more trivial {\bf Maker win stop} condition: the crossing of two 2-lines.

\begin{theorem} 
\label{maker_win_condition}
Supposing an optimal breaker, maker can win if and only if he moves to the crossing of two 2-lines.
\end{theorem}

\begin{proof}
The if case is trivial since moving into the crossing of two 2-lines yields two 1-lines. To prove the only if part, assume for contradiction that maker wins without ever moving into the crossing of two 2-lines. Each time maker moves into a 2-line, the optimal breaker strategy responds by moving into the resulting 1-line. Since it never happens that more than one 1-line is created, breaker can always break all 1-lines, so maker cannot win. This contradiction proves the theorem.
\end{proof}

\subsection{Eliminating branches from the search space}
We identify situations where we can safely restrict the valid moves. Proving that these restrictions do not affect the game theoretic value of the board is 
%straightforward and is 
left to the reader.

\begin{itemize}
    \item {\bf Forced move}: if a board contains a 1-line or the crossing of two 2-lines (Theorem~\ref{maker_win_condition}), we are allowed to disregard all other moves. 
    \item {\bf Dead square elim}: if an empty square is not contained in any lines, then we can eliminate that square, since neither player benefits from moving there.
    %\item {\bf Lone 2-line}: When two squares form a 2-line and no other line contains them, then we can eliminate them by pretending that maker moves into one of them, forcing breaker to move to the other. This is a special case of partial pairing.
    \item {\bf Dominated square}: Suppose square $s$ is contained in a single line $l$. If there is another square $s'$ that is only contained in $l$, then they form a partial pairing and $s, s', l$ can be eliminated. If $l$ is a 2-line, then its other square $s'$ dominates $s$ and we can always assume that maker eagerly moves to $s'$, forcing breaker to move to $s$.
    %\item {\bf Lone 2-line}: We can eliminate an l-line containing two squares, $a, b$ such that their degree is one, because if maker can win, it can win without this line. This is a special case of partial pairing.
    %We can suppose when maker moves to this line, it will move to the 1-degree fields in the end. If this holds, the line is not connected to other lines. Now it's easy to see, that if in the remaining part maker doesn't have the maker win condition (\ref{maker_win_condition}), after adding this line, still will not have.
    %\item {\bf Half lone 2-line}: Suppose the board contains a 2-line with squares $a, b$ such that $a$ is not contained in any other line. When maker moves, we can assume that he moves to $b$, forcing breaker to move to $a$.
\end{itemize}

\subsection{Avoiding repeated searches}
The same set of moves, played in different orders result in identical boards. We can save a lot of computation and memory by maintaining a transposition that maps boards to search nodes, as we can use the same search node for identical boards. This turns the search tree into a search DAG. Furthermore, we also exploit the horizontal symmetry of the game, i.e., collapse states that are symmetrical. 
We introduce an even more refined transposition table which exploits the isomorphy of boards (considered as hypergraphs). We transform each graph into a canonical form and store it in the transposition table. We refer to this extension as {\bf Isomorphy}.

\subsection{Partitioning the board}

Consider a hypergraph $(V, E)$ with marks of maker and breaker at $V_M, V_B \subset V$, respectively (with $V_M \cap V_B = \emptyset$). We define the {\bf residual hypergraph} $(V', E')$ as follows. $V' = V \setminus V_M \setminus V_B$ and $E' = \big\{ e \cap V' \ |\ e \in E,\ e \cap V_B = \emptyset \big\}$. When we continue playing in $(V, E, V_M, V_B)$, then it is equivalent to starting a new game in $(V', E')$.

Theorem~\ref{thm:reduction} implies that if the residual graph is not connected, then we only need to find the game-theoretic values of the components. We can use a similar tool if the hypergraph is not 2-connected.  
\begin{theorem}
\label{thm:2-connected}
Let $H = (V, E)$ a (possibly infinite) hypergraph with subhypergraphs $(V_1, E_1)$ and $(V_2, E_2)$ satisfying $V_1 \cup V_2 = V$, $V_1 \cap V_2 = \{v\}$ and $E = E_1 \cup E_2$. If maker starts the game, then he can win in $(V, E)$ if and only if one of the following holds.
\begin{enumerate}
    \item Maker can win $(V_1, E_1)$.
    \item Maker can win $(V_2, E_2)$.
    \item Maker can win both games $(V_1, E_1)$ and $(V_2, E_2)$ with the extra advantage that $v$ is colored with maker's color and maker can still make the next move.
\end{enumerate}
\end{theorem}

\begin{proof}
If 1) or 2) holds, then maker wins $(V, E)$ by winning $(V_1, E_1)$ or $(V_2, E_2)$. If 3) holds then maker chooses $v$ and then wins the game $(V_i, E_i)$ in which breaker does not respond in his next move.

If none of the three conditions hold, then by symmetry, we can assume that breaker can win $(V_1, E_1)$ (played normally) and breaker wins $(V_2, E_2)$ even if maker has the extra move at $v$. Breaker can follow these strategies as in Theorem~\ref{thm:reduction}. (In $(V_2, E_2)$, breaker assumes having maker's mark at $v$.)
\end{proof}

\begin{theorem}
\label{thm:2-edge-connected}
Let $H = (V, E)$ be a (possibly infinite) hypergraph with subhypergraphs $(V_1, E_1)$ and $(V_2, E_2)$ satisfying $V_1 \cup V_2 = V$, $V_1 \cap V_2 = \emptyset$ and $E = E_1^{\prime} \cup E_2^{\prime} \cup \{e\}$ and $E_i = E_i^{\prime} \cup \{e \cap V_i\}$. If maker starts the game, then he can win in $(V, E)$ if and only if one of the following holds.
\begin{enumerate}
    \item Maker can win $(V_1, E_1^{\prime})$.
    \item Maker can win $(V_2, E_2^{\prime})$.
    \item Maker can win both games $(V_1, E_1)$ and $(V_2, E_2)$.
\end{enumerate}
\end{theorem}

\begin{proof}
If one of them holds, then maker can win $(V, E)$ by following the winning strategy or strategies as in Theorem~\ref{thm:reduction}.

If none of them holds, then by symmetry, we can assume that breaker can win both $(V_1, E_1)$ and $(V_2, E_2^{\prime})$, and hereby he can win $(V, E)$.
\end{proof}

A linear time algorithm can be used to detect if the hypergraph is not 2-connected. In such cases, the board can be reduced to 4 smaller boards, according to Theorems \ref{thm:2-connected} and \ref{thm:2-edge-connected} that can be evaluated independently. We refer to this optimisation as {\bf Components}.

%Whether there is a non-trivial partition of the hyprgraph as in Theorem \ref{thm:2-connected} or \ref{thm:2-edge-connected} can be tested in linear time. Namely, we find the 2-connected components of the incidence graph of the hypergraph.
%Each time a new search node is created, we check if the corresponding hypergraph is not 2-connected. If so, we reduce the board into 4 smaller boards, according to Theorems \ref{thm:2-connected} and \ref{thm:2-edge-connected} and evaluate them independently. We refer to this optimisation as {\bf Components}.

\subsection{Replacing \pn/\dn values with game specific heuristics} \label{Heuristic_PN_DN}

PNS maintains \pn/\dn values for each search node, tracking the number of leaves that need to be proven or disproven to solve the given node. These values determine the next leaf to expand.
However, by setting leaf \pn/\dn values to $1$, this technique disregards the fact that two leaf nodes can be hugely different, due mostly to two factors: 1) the winner might be much more apparent in one situation than in another and 2) boards with many colored squares are easier to evaluate as they are closer to the end of the game. In the following, we explore the benefit of replacing \pn/\dn values with heuristic board evaluation functions in the leaf nodes~\footnote{\cite{Allis1994SearchingFS,mcts_saito,PNS_variants} also explore alternative initialization techniques.}.

%
%\begin{figure}
%    \centering
%    \includegraphics[width=0.4\textwidth]{figures/potential.png}
%    \caption{Stepwise potential in an illustrative game. Maker moves increase, while breaker moves decrease potential. In most reasonably "close" games, the average potential trends downwards.}
%    \label{fig:potential}
%\end{figure}

%\begin{figure}[htb]
%    \centering
%    \includegraphics[width = 0.4\textwidth]{papers/cog2021/disp_set.png}
%    \caption{Disproof setup: 2 weak attacker moves, and one strong defender move}
%    \label{fig:disproof_setup}
%\end{figure}

\begin{figure}%
    \centering
    %\subfloat[\centering Disproof setup]
    %\vspace{10pt}
    {{
    \includegraphics[width = 0.21\textwidth]{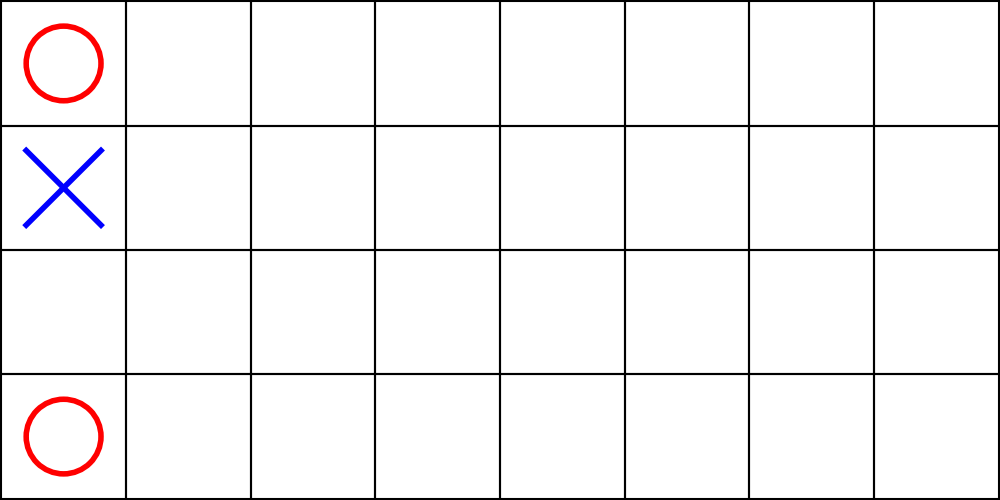}
    %\vspace*{2mm}
    }}
    \qquad
    %\subfloat[\centering Potencial changing]
    {{
        \includegraphics[width=0.21\textwidth]{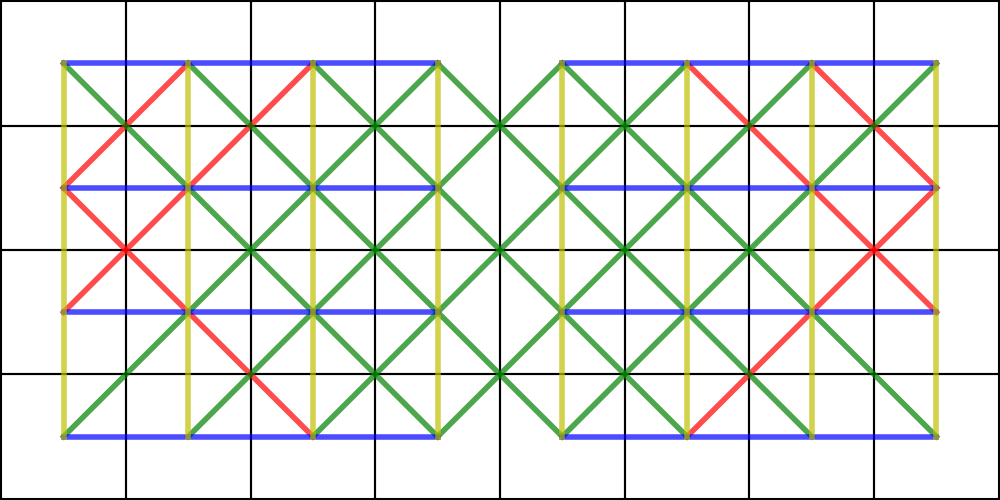}
    }}
    \caption{(Left) Disproof setup: 2 weak attacker moves, and one strong defender move (Right) The visualization of the hyperedges shorter than 7 that are added to each partition defined in Section~\ref{sec:reduction}.}%
    \label{fig:disproof_setup_and_lines}%
    \vspace{-10pt}
\end{figure}

\begin{figure}
    \vspace{-20pt}
    \centering
    \includegraphics[width=0.4\textwidth]{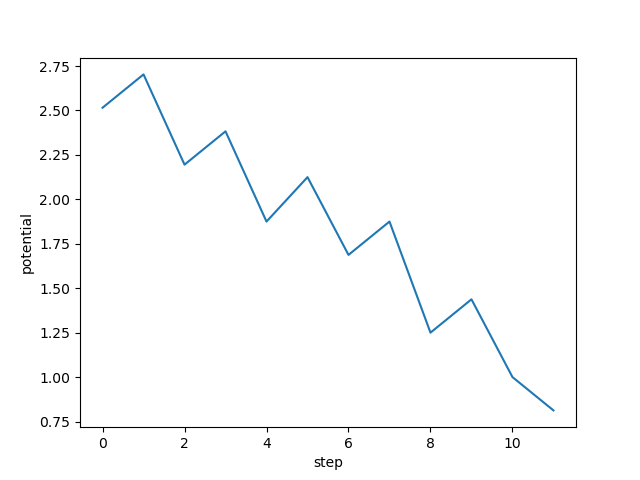}
    \caption{Stepwise potential in an illustrative game. Maker moves increase, while breaker moves decrease potential. In most reasonably "close" games, the average potential trends downwards.}
    \label{fig:potential}
    \vspace{-10pt}
\end{figure}

We know that once the potential of the board goes below a threshold ($1$ before breaker move), breaker has won the game. Manual inspection of game plays reveals that strong breaker moves are often those that greatly decrease potential. Figure~\ref{fig:potential} shows how potential changes in a typical close game. Potential values are monotonic decreasing if we consider OR and AND nodes separately. We search for a heuristic \dn function in the form of $\alpha^{pot(b)}$ and select $\alpha = 1000$ based on a grid search on the values $\{2, 10, 20, 100, 200, 1000, 2000\}$. Note however, that it makes little sense to compare potential values across OR and AND nodes, as OR node values are systematically lower, since the last move was by breaker. Hence, for OR nodes we consider the potential of the parent (an AND node), which is inherited by the child with the smallest potential. All other children are adjusted with the potential difference relative to the smallest child. Given board $b$ with parent $p$ and sibling $s$ such that $s$ has the smallest potential among the children of $p$, our {\bf heuristic $\dn$} value is

$$
DN(b) =
\left\{
	\begin{array}{ll}
		1000^{pot(b)} & \mbox{b is an AND node}\\
		1000^{pot(p) - pot(s) + pot(b)}  & \mbox{b is an OR node}
	\end{array}
\right.
$$

The applicability of the potential function is less straightforward in replacing the \pn function. This is because maker typically wins well before all squares are colored, and it might have many short wins that is not captured in aggregate line information. Original \pn is good at capturing short wins as such branches will have less leaves, i.e., lower \pn values. Hence, instead of replacing \pn values, we are looking to adjust them with game specific knowledge. We do this by accumulating search data and fit a model to it. 

We run PNS on $4 \times n$ boards for $n \in \{7, 8, 9, 10\}$ and collect states whose proof theoretic values have been proven. To obtain a balanced training set, we use two setups. The \emph{proof} setup starts from the empty board $b$ which is maker win for these $n$ values. The \emph{disproof} setup starts from a board $b'$ which contains two weak maker moves and one strong breaker move, shown in Figure~\ref{fig:disproof_setup_and_lines}. This initialization changes the game theoretic value, i.e., breaker wins. Hence we collect data both from successful proof and disproof searches. This yields $11076040$ board positions.

\begin{figure}
    \centering
    \includegraphics[width=0.24\textwidth]{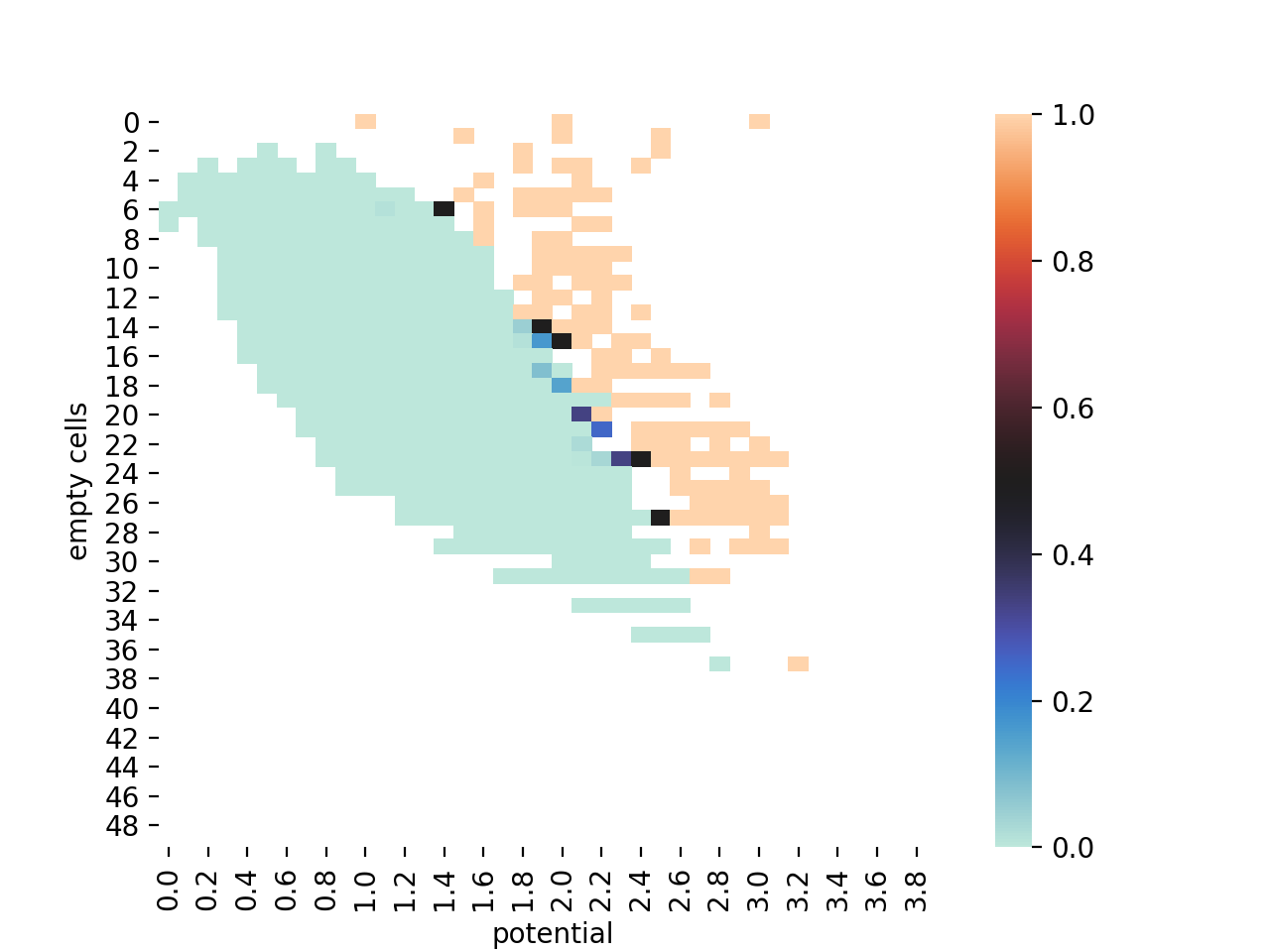}
    \includegraphics[width=0.24\textwidth]{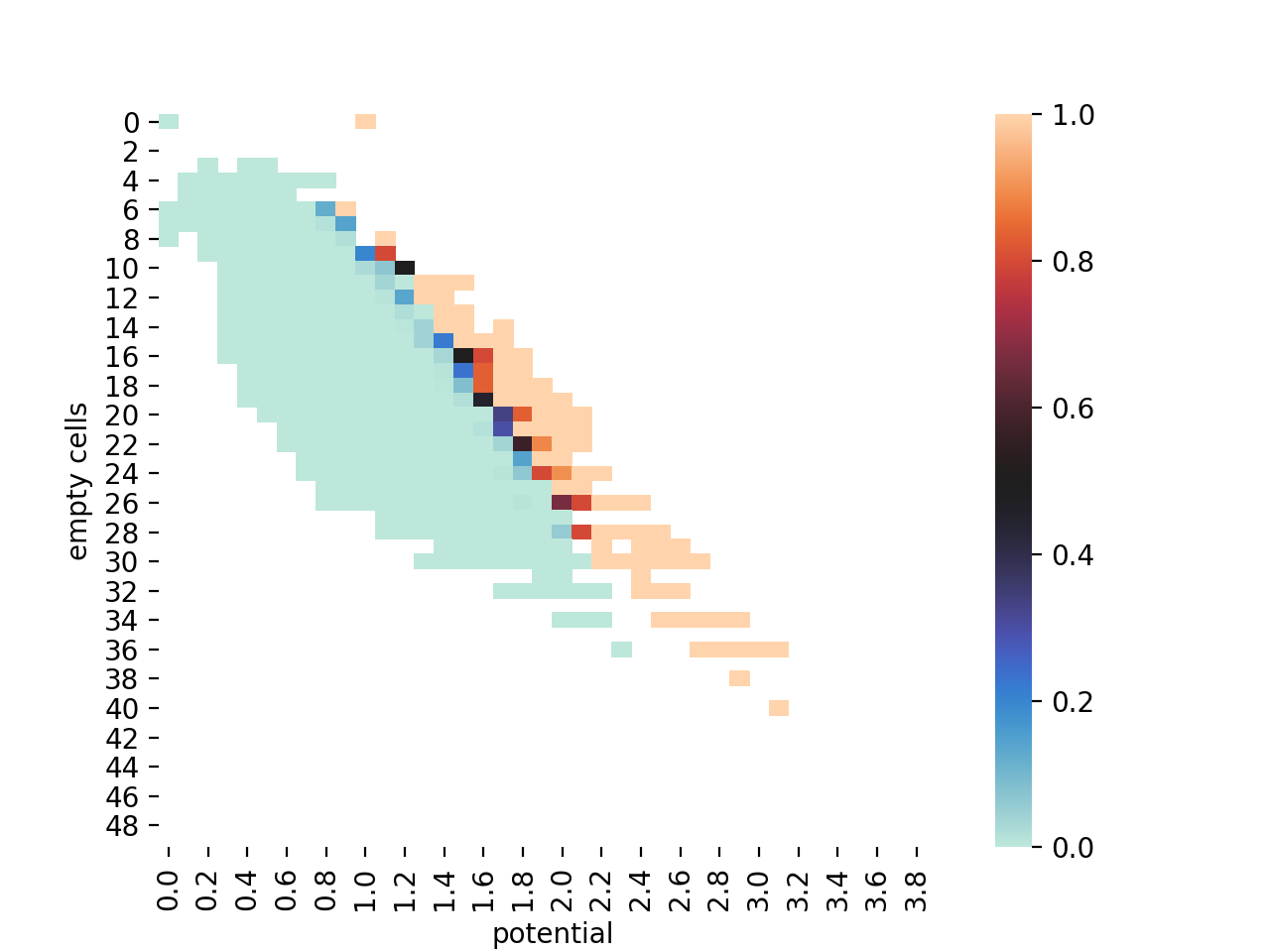}
    \caption{Game theoretic value heatmap as a function of potential (x-axis) and the number of remaining empty cells (y-axis). {\bf Left}: OR nodes (maker moves). {\bf Right}: AND nodes (breaker moves).}
    \label{fig:pns_win_boundary}
    \vspace{-10pt}
\end{figure}

Figure~\ref{fig:pns_win_boundary} plots the game theoretic value as a function of the board potential and the number of remaining empty squares. On both plots, there is a well visible separating plane between maker/breaker win configurations. We estimate the probability of breaker win by fitting a logistic regression curve to this data. Our {\bf heuristic $\pn$} function is obtained by adding this probability to the leaf \pn value.

\begin{align*} 
\vspace{-20pt}
logOdds(b) = & -6.2 - 13.4 \cdot nodeT(b) \\
& - 1.52 \cdot emptyS(b) + 25.83 \cdot pot(b) \\
probBwin(b) = & 1- \frac{1}{1+e^{-logOdds(b)}} \\
pn(b) = & 1 + \beta \cdot probBwin(b)
\end{align*}

\noindent where $nodeT(b)$ is $0$ for AND nodes and $1$ for OR nodes and $emptyS(b)$ is the number of empty squares. The value of $\beta$ is set to $10$ based on a small grid search in $[1, 1000]$.

\section{Experiments}
\label{sec:experiments}

We evaluate our techniques introduced in Section~\ref{sec:pns} separately, as well as jointly on various $(4,n,7^{tr})$ boards. Our two evaluation metrics are 1) \emph{Size} which refers to the number of nodes created during search and 2) \emph{Time} which is the total search time in seconds. We enforce a 1 hour time limit and a 60GB memory limit on each experiment.

In Table~\ref{tab:all_technique} we evaluate the effect of each technique separately on a small board of size $4 \times 7$. Our baseline is PNS extended with a symmetry aware transposition table: we include the transposition table in the baseline because without it PNS quickly runs out of memory even on very small boards. We observe that all techniques, except for isomorphy bring significant improvement both in terms of time and search space size. The most powerful method is forced move, which alone reduces the search space by three orders of magnitude. Note that the component checking algorithm requires that dead squares are eliminated, hence the line corresponding to components contains dead square elim as well.

Checking for isomorphy reduces the search space by around one order of magnitude, however, it brings extra computation that makes the overall search slower. We find that the greater the boards, the less frequently we find isomorphic states, hence we decided to remove isomorphy from later experiments to avoid the added time penalty. Identifying situations where it is still worth checking for isomorphy deserves further analysis which we leave for future work.

% Trick dn does not improve vanilla ==> PN & DN
% Components cannto exist without Dead_fields ==> Comp & Dead_fields

\begin{table}[htbp]
\caption{The effect of PNS enhancements on search space and search time, on board $4\times 7$, proof setup}
\begin{center}
\begin{tabular}{ c | c c }
Method & \textbf{Time (s)} & \textbf{Size} \\
\hline 
Baseline PNS & 4.49 & 1921106  \\
Forced move & 0.005 & 822  \\
Dead square elim & 0.98 & 332740  \\
Dominated squares & 0.34 & 116991  \\
%Half lone 2-line & 0.47 & 116991  \\
Breaker win stop & 1.78 & 660763  \\
Heuristic \pn and \dn & 1.03 & 481553  \\
%Heuristic \pn & 0.44 & 168862.00  \\
%Heuristic \dn & 20.81 & 7727239.00  \\
Components & 0.69 & 167514  \\
Isomorphy & 18.47 & 249348  \\
%Together & 0.05 & 970.00  \\
\hline
\end{tabular}
\label{tab:all_technique}
\end{center}
\end{table}

We call \emph{PNS+} the variant with all techniques except for isomorphy. To better assess the quality of each technique, we run PNS+ on larger boards and we check what happens if the optimisations are removed one by one. The results are shown in Table~\ref{tab:ablation-1-2}, both in proof and disproof setups. Instead of raw time and size values, we indicate ratios with respect to PNS+ to emphasize the performance contribution. In general, leaving out one of the heuristics slows down the proof setup much more than the disproof setup. Designing good heuristics is usually easier for the proof setup, as the disproof setup has much weaker stop-conditions.

% \subsection{Experiment 2}

%Without adding this techniques the vanilla-PNS is not able to prove even the $4x9$ board.
%The contrbution of this techniques can be measured by ablation from the best setup.
%Table \ref{tab:ablation-1-2} shows, that using transposition table, and forced move is essential in order to prove the $4x12$ board, without any of them it is not possible.

%In general, we can see, that leaving out one of the heuristics slows down mainly the proof setup.
%Defining a technique, which speeds up a proof is much easier, e.g.: moving the right step could turn out quickly (by collecting a line).
%However the disproof stopping conditions are deeper in the tree, and the feedback for a good move turns out later for a technique, which wants to speed up a disproof.
%That's why in the disproof setup our technique's gain is not so spectacular.

\begin{table}[htbp]
\caption{The penalty associated with removing each heuristic. Numbers are ratios with respect to PNS with all heuristics. Values marked with * reached either time or memory limit.}
\begin{center}
\begin{tabular}{ c | r r | r r }
\hline
& \multicolumn{2}{c}{$n=11$} & \multicolumn{2}{c}{$n=12$} \\
& \textbf{Time} & \textbf{Size} & \textbf{Time} & \textbf{Size}\\
\hline 
& \multicolumn{4}{c}{\bf {Proof Setup}} \\
PNS+ & 1.00 & 1.00  & 1.00 & 1.00  \\
Components & 2.69 & 3.01  & 1.30 & 1.61  \\
%Transposition table & 71.00 & 116.70  & * & * \\
Breaker win stop & 0.75 & 1.00  & 1.32 & 1.74  \\
Dead square elim & 2.45 & 2.96  & 3.07 & 4.10  \\
%Half lone 2-line & 3.58 & 3.58  & 14.23 & 15.56  \\
Dominated squares & 3.58 & 3.58  & 14.25 & 15.56  \\
Forced move & 85.67 & 69.50  & * & *  \\
%Heuristic \pn & 13.38 & 11.48  & 17.62 & 15.80  \\
%Heuristic \dn & 11.89 & 10.18  & 8.32 & 8.00  \\
Heuristic \pn and \dn & 27.82 & 23.92  & 43.66 & 41.96  \\

%\hline
%Together (raw) & 0.91 (s) & 191612 & 7.83 (s) & 1341134 \\

\hline 
& \multicolumn{2}{c}{$n=8$} & \multicolumn{2}{c}{$n=9$} \\
& \textbf{Time} & \textbf{Size} & \textbf{Time} & \textbf{Size}\\
\hline 

& \multicolumn{4}{c}{\bf {Disproof Setup}} \\
PNS+ & 1.00 & 1.00  & 1.00 & 1.00  \\
Components & 0.93 & 1.14  & 0.88 & 1.05  \\
%Transposition table & 3.57 & 7.45  & 5.88 & 16.16  \\
Breaker win stop & 1.01 & 1.03  & 0.97 & 0.99  \\
Dead square elim & 1.52 & 2.30  & 1.39 & 2.41  \\
%Half lone 2-line & 1.43 & 2.00  & 1.94 & 2.04  \\
Dominated squares & 1.44 & 2.00  & 2.05 & 2.04  \\
Forced move & 2.37 & 2.38  & 3.50 & 3.31  \\
%Heuristic \pn & 1.02 & 1.07  & 1.08 & 1.22  \\
%Heuristic \dn & 2.59 & 3.59  & 1.37 & 2.36  \\
Heuristic \pn and \dn & 3.23 & 4.32  & 1.11 & 1.92  \\
%Together (raw) & 3.34 (s) & 578881 & 82.80 (s) & 9606929 \\

\hline
\end{tabular}
\label{tab:ablation-1-2}
\end{center}
\end{table}

%\subsection{Experiment 3}

We use PNS+ configuration and explore what $4\times n$ board configurations we can solve with it. Within our time and memory limit, PNS+ proves that maker wins for $n \leq 14$. The left part of Figure~\ref{fig:proof_results} shows that the search space and required time grows exponentially, in line with our expectations. A similar trend can be observed in the disproof setup, i.e., when breaker wins (right side of Figure~\ref{fig:proof_results}), potentially with a bit slower increasing curve.

We find that the general pattern for maker win is to start from the left and right sides where the extra short lines pose serious threats. As breaker contains these threats, maker's position gets weaker towards the center, but as the two sides meet, it can combine the threats and win there. The larger the $n$, the harder this is for maker.

\begin{figure}
    \centering
    \includegraphics[width=0.38\textwidth]{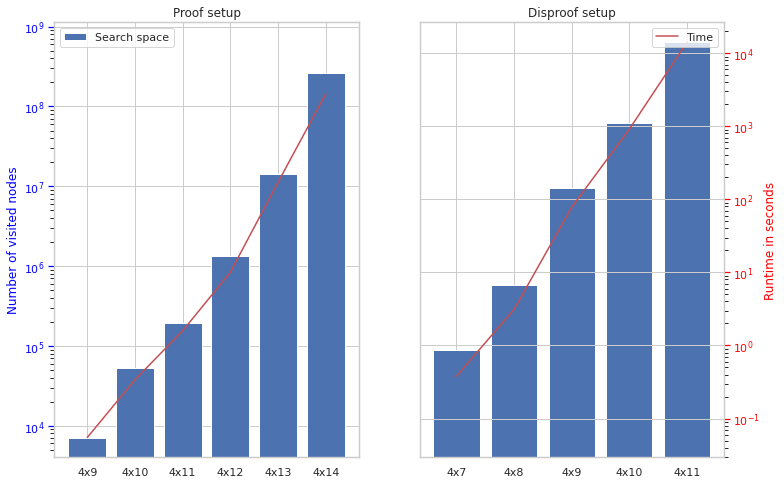}
    \caption{Exponentially increasing runtime and search space size in proof and disproof setups. The logarithmic curve is steeper in proof setup: the gradient of the fitted line is $2.78$ vs $2.74$.}
    \label{fig:proof_results}
    \vspace{-10pt}
\end{figure}

%\subsection{Experiment 4}

PNS+ proves maker win for all boards on which it terminates, hence it remains hypothetical that increasing $n$ tilts the game towards breaker. To get an estimate of where the turning point might be, we aim to quantify the difficulty of maker win for different $n$ values.

Let us consider the set of all descendants of the root node that are all AND nodes, such that neither of their ancestors are AND nodes. This includes the children of the root, as well as all children of nodes that flipped from AND to OR due to board simplification (via forced moves). We call this set the \emph{support} of the root, as the value of the root depends on this set in that the root is a maker win exactly if one element of the support is a maker win. In Table~\ref{tab:balance} we show what fraction of the root support is breaker win, as $n$ increases. We observe that this fraction is increasing, which supports the conjecture that for some $n$ all nodes in the root support will turn to breaker win, making the whole game breaker win. 

For this last experiment, we impose a 10GB memory limit, as a result of which PNS+ fails on the $4\times 12$ board for 158 starting positions. For these, we repeat the experiment with 125GB, but still get 14 failures. This is why Table~\ref{tab:balance} shows intervals for this board size. Note, however, that maker win positions are typically much faster to prove, so the failed positions are likely breaker win. Hence we conjecture that the values are closer to the higher end of the interval.

\begin{table}[htbp]
\caption{The balance between maker win and breaker win nodes in the support set of the root for different $n$}
\begin{center}
\begin{tabular}{ c | c c c }
\hline
\textbf{Board} & \textbf{Support size} & \textbf{Breaker win}  & \textbf{Breaker win \%}\\
\hline 
4x7 & 308 & 114 & 37\% \\
4x8 & 314 & 150 & 48\% \\
4x9 & 371 & 156 & 42\% \\
4x10 & 415  & 200 & 48\% \\
4x11 & 475  & 248 & 52\% \\
%4x12 & 515  & 364(158) & 70.67 \% \\
4x12 & 515  & 290(+14 fail) & 56-59\% \\
\hline
\end{tabular}
\label{tab:balance}
\end{center}
\vspace{-10pt}
\end{table}

\section{Conclusion and Future Work}
Our research aims to prove the longstanding conjecture that the 7-in-a-row game is a draw and our paper presents progress towards this proof. We reduce the original game into a small, finite maker-breaker game called $(4,n,7^{tr})$--game, for some arbitrary $n$. We explore Proof Number Search for solving this finite variant and introduce various heuristics to make Proof Number Search more efficient. Our experimental results indicate that maker wins for small $n$ values, however, as $n$ increases, it gets harder for maker. We expect that there is a turning point, i.e. a $n_0$ value such that breaker wins for $n \geq n_0$. However, our current PNS architecture cannot yet scale to large enough $n$ values.

There are several promising directions to improve our results. Some existing search heuristics can be directly encorporated into our system, such as pairing strategies, partial pairings and relevance-zones. 
%It is yet unclear whether these methods will bring improvement on top of our existing methods. 
Another possible direction is to relax strategy partitioning of breaker and allow for some cooperation between different boards. Preliminary experiments suggest that such cooperation makes the game easier for breaker, bringing the turning point closer to what is computationally feasible.

\section*{Acknowledgment}
We are grateful to D\'{a}niel L\'{e}vai for reviewing and optimising our PNS codebase. We also thank Levente Kocsis for his valuable suggestions related to Proof Number Search.

\bibliography{gobang}
\bibliographystyle{plain}

\end{document}